%% file: arxiv.tex
\documentclass[10pt]{article}
\usepackage{graphicx} 
\usepackage{amsmath,amssymb,amsfonts,amsfonts,amsfonts,hyperref, amsthm, MnSymbol, verbatim}
\usepackage[top=2.5cm, bottom=2.9cm, left=2.5cm,
right=2.5cm]{geometry}
\newtheorem{definition}{definition}

\newtheorem{proposition}{proposition}

\newtheorem{theorem}{Theorem}

\newtheorem{lemma}{Lemma}
\newtheorem{assumption}{Assumption}

\newtheorem*{remark}{Remark}
\numberwithin{equation}{section}

\renewenvironment{proof}{\noindent {\bfseries\rmfamily Proof.}}{\qed}
\hypersetup{hidelinks, colorlinks = true, linkcolor=blue, anchorcolor=red,
	citecolor=green}
\newcommand{\ie}{i.e., }
\newcommand{\R}{\mathbb{R}}
\newcommand{\N}{\mathbb{N}}

\newcommand{\F}{\mathcal{F}}
\newcommand{\E}{\mathcal{E}}
\newcommand{\X}{\mathcal{X}}
\newcommand{\Y}{\mathcal{Y}}
\newcommand{\Z}{\mathcal{Z}}

\newcommand{\calP}{\mathcal{P}}

\newcommand{\bE}{\mathbb{E}}

\newcommand{\calH}{\mathcal{H}}
\newcommand{\jd}{\vspace{10pt}}

\newcommand{\biavg}{\frac{1}{n(n-1)} \sum_{i\neq j}^{n}}

\title{Generalization analysis with deep ReLU networks for metric and similarity learning}
\author{Junyu Zhou$^{1*} $\quad Puyu Wang$^{2*} $\quad Ding-Xuan Zhou$^3$ \vspace{5mm} \\ 
$^{1}$ Catholic University Eichstätt-Ingolstadt, Germany \\
$^{2}$ RPTU Kaiserslautern-Landau, Germany  \\
$^{3}$  University of Sydney, Australia 
}
\begin{document}
\date{}
\maketitle
\begin{abstract}
   While metric and similarity learning has been extensively studied from several theoretical perspectives, a rigorous understanding of its generalization performance is still lacking. In this paper, we investigate the generalization behavior of metric and similarity learning by exploiting the \textit{specific structure} of the true metric (i.e., the target function). In particular, by deriving the explicit form of the true metric for metric and similarity learning with the hinge loss, we construct a structured deep ReLU neural network as an approximation of the true metric, whose approximation ability depends on the network complexity. Here, the network complexity is characterized by the network depth, the number of nonzero weights, and the number of computational units. Based on the hypothesis space consisting of such structured deep ReLU networks, we establish excess risk bounds for metric and similarity learning by carefully controlling both the \textit{approximation error} and the \textit{estimation error}.
   An explicit excess risk rate is derived by choosing the proper capacity of the constructed hypothesis space. To the best of our knowledge, this is the \textit{first} generalization analysis that provides explicit excess risk bounds for metric and similarity learning. In addition, we investigate properties of the true metric for metric and similarity learning under more general loss functions. Experiments show that the proposed model is empirically competitive and better captures the underlying similarity structure.\footnote{The first two authors contributed equally.} 
\end{abstract}

\bigskip

\input{introduction}

\input{deepReLU}
\input{Regularites}
\input{conclusion}

\section*{Acknowledgments} The work of Puyu Wang is supported by the Alexander von Humboldt Foundation. The work of Ding-Xuan Zhou is partially supported by the Australian Research Council under project DP240101919.

\bibliography{sample}
\bibliographystyle{plain}
\end{document}

%% file: introduction.tex
\section{Introduction}

    Metric and similarity learning aims to study a metric $d$ from an observed sample $S$ that estimates the distance or the similarity between a pair of observers $(x,x')$. The observed sample $S=\{Z_i=(X_i, Y_i)\}_{i=1}^n$ is assumed to be independently drawn from an unknown distribution $\rho$ defined over the input-output space $\Z=\X\times\Y$, where $\X\subset\R^p$ is closed and bounded and $\Y = \{y_1,...,y_m\} \subset \R$ is the set of labels. Let $\tau(y, y')$ be the reducing function defined by $\tau(y, y') = 1$ if $y = y'$ and $ \tau(y, y')= -1$ else. The performance of $d$ on a pair $(z,z')$ is usually measured by $\mathbf{1}_{[\tau(y, y')d(x, x') > 0]}$, where $\mathbf{1}_{[A]}$ is the indicator function taking value $1$ if the event $A$ happens and $0$ otherwise. 
    Since the indicator function  $\mathbf{1}_{[\tau(y, y')d(x, x')  > 0]}$ is non-convex and discontinuous, one often considers replacing it with a convex surrogate loss $\ell: \R \to \R^+$. Given by $\ell\left(\tau(y, y') d(x, x')\right)$, the generalization error (true risk) associated with a metric $d$ is  defined as
	\begin{align*}
		\E(d) &=  \bE_{Z, Z'}\left[ \ell\big(\tau(Y, Y') d(X, X')\big)\right]\\
        &= \int_{\Z\times\Z} \ell\big(\tau(y,y')d(x,x') \big) d\rho(z)d\rho(z').
	\end{align*}
 The corresponding empirical error based on the sample $S$ is defined as
	\begin{equation*}
		\E_S(d) = \biavg \ell\big(\tau(Y_i, Y_j) d(X_i, X_j) \big).
	\end{equation*}

Let $\hat{d}_z = \arg\min_{d \in \calH} \E_S(d)$ be the minimizer of the empirical error over a hypothesis space $\calH$, and $d_\rho=\arg\min \E(d)$ be the true metric that minimizes the generalization error over the space of all measurable functions on $\X\times\X$.
The statistical generalization performance of $\hat{d}_z$ is measured by the excess generalization error $\E(\hat{d}_z) - \E(d_\rho)$, i.e., the distance between the expected risk $\E(\hat{d}_z)$ and the least possible risk $\E(d_\rho)$. 
To examine the excess generalization error, we use the following error decomposition:  
    \begin{align}\label{eq:decomp}       \E(\hat{d}_z) - \E(d_\rho) = \{\E(\hat{d}_z) - \E(d_\calH)\} + \{\E(d_\calH) - \E(d_\rho)\},    \end{align}
    where $d_\calH=\arg\min_{d\in\calH}\E(d)$ is the minimizer of the generalization error over the hypothesis space $\calH$. The terms $\E(\hat{d}_z) - \E(d_\calH)$ and $\E(d_\calH) - \E(d_\rho)$ in \eqref{eq:decomp} are called the estimation error and the approximation error, respectively. 
  
There are a considerable amount of theoretical works on metric and similarity learning \cite{cao,davis2007information,guo2014guaranteed,jin2009,kar2011similarity,lei2016generalization,maurer2008learning, tatlimetric, wu2025stability, fast2019}, 
most of which focused on learning the Mahalanobis distance $d_M(x,x')=(x - x')^\top M(x - x')$ for metric learning \cite{cao,davis2007information,jin2009, lei2016generalization,fast2019} and the pairwise similarity function $s_M(x,x') = x^\top Mx'$ for similarity learning \cite{cao,chechik2010large,guo2014guaranteed,kar2011similarity,maurer2008learning,shalit2010online}. Here, $M\in\mathbb{S}^+_p\subset \R^{p\times p}$ is a positive semi-definite matrix. The corresponding estimation error has been studied in \cite{cao,guo2014guaranteed,deep2019huai,fast2019}. Specifically, \cite{cao} studied the estimation error with the hypothesis space $\calH=\{(x-x')^\top M(x-x') - b: M\in\mathbb{S}^+_p,b\in\R^+\}$ for metric learning with the hinge loss. With the same hypothesis space, \cite{fast2019} derived fast learning rates of the estimation error for metric learning with smooth loss function and strongly convex objective. \cite{guo2014guaranteed} investigated the estimation error with the hypothesis space $\calH = \{b - x^\top Mx':M\in\mathbb{S}^+_p,b\in\R^+\}$ for similarity learning.
However, the expressive power of the studied hypothesis spaces is very limited.  Indeed, functions in the Mahalanobis distance and the pairwise similarity can be viewed as a multivariate polynomial with order $2$ on $\R^p$, then the dimension of the linear span of the hypothesis space $\dim(span(\calH))$ is controlled by $p^2$. It implies that the dimension $p$ of the input space determines the complexity of $\calH$ and further limits the expressive ability of $\calH$. If the expressive ability of $\calH$ is not enough, there will be a situation where the estimation error converges to $0$ fast while the approximation error is very large. To fully understand the generalization performance of $\hat{d}_z$, it is necessary to study the approximation error in addition to the estimation error. To the best of our knowledge, no study provides  estimates of the approximation error for metric and similarity learning problems, due to the fact that the explicit form of the true metric is unknown. 

In this paper, we address the gap in the theoretical understanding of the generalization behavior of metric and similarity learning with the hinge loss by exploiting the specific structure of the true metric.
Our main contributions are summarized as follows.
    \begin{itemize} 
        \item We provide a comprehensive generalization analysis for metric and similarity learning with the hinge loss. In contrast to previous works \cite{cao,guo2014guaranteed,deep2019huai,fast2019}, which focus only on the estimation error, we study both the estimation error and the approximation error, and thereby derive the \textit{first} explicit excess risk bounds for metric and similarity learning. Under mild conditions, we further establish a refined learning rate up to a logarithmic factor,
        $O(n^{-\frac{(\theta + 1) r}{p + (\theta + 2)r}})$, where $p$ is the dimension of the input space, $\theta$ is the parameter in the noise condition, and $r$ is the smoothness index of the conditional probabilities.
        \item A key technical contribution of this work is the construction of a structured deep ReLU neural network $F_a(1-2\sum_{i=1}^m \phi(h_i(x),h_i(x')))$ to approximate the true metric $d_\rho$, based on the representation  $d_\rho(x,x')=sgn(1-2\sum_{i=1}^m p_i(x)p_i(x')) $ for $x,x'\in\X$  (see Theorem~\ref{true_predictor}), where $p_i(x)=Prob\{Y=y_i\mid X=x\}$ and $a>0$ is a free parameter. This representation shows that, if one can design a family of sub-networks that approximate $p_i(x)$ well for each $i=1,\ldots,m$, then one can obtain a good approximation of $d_\rho(x,x')$ by combining a deep ReLU network $\phi(x,y)$ for approximating the multiplication function $xy$ with a two-layer ReLU network $F_a(x)$ for approximating the sign function $sgn(x)$. Our results (Theorems~\ref{approximation_error}, \ref{thm:estimationerror}, and \ref{thm:excess}) show that the network with the constructed form can achieve favorable approximation and estimation error bounds for metric and similarity learning.

        \item We also investigate structural properties of the true metric and the problem formulation for metric and similarity learning under general loss functions.
        In particular, we show that the bias term appearing in many existing works \cite{cao,jin2009,fast2019} is not intrinsic to the theoretical characterization of the true metric, and it is natural to focus on the case where the output space $\Y$ consists of finitely many labels rather than contains a continuous interval.
        Moreover, we show that the true metric is symmetric, and under mild and intuitive conditions, the true metric between two identical samples is always less than or equal to that between two different samples. This provides further theoretical justification for the use of symmetric models such as the Mahalanobis distance in metric learning.
\end{itemize}

The remainder of the paper is organized as follows. In Section~\ref{sec:generalization}, we present generalization bounds for metric and similarity learning based on deep ReLU networks. Section~\ref{ssec:regularity} investigates structural properties of the true metric, and Section~\ref{sec:exper} empirically validates our theoretical findings. Finally, Section~\ref{sec:conclu} contains concluding remarks.

%% file: deepReLU.tex
\section{Generalization analysis with deep ReLU networks}\label{sec:generalization}
    We begin by introducing some notations that will be frequently used in the rest of the paper. Let $Z = (X,Y)$ and $Z' = (X',Y')$ be random variables independently following $\rho$, and $\rho(\cdot|x)$ be the conditional distribution of $Y$ given $X=x$. We denote the conditional probability of $Y = Y'$ given $X = x, X' = x'$ as 
    \begin{align}\label{eta}
	 	\eta(x, x') = Prob\{Y = Y'|X = x, X' = x'\} = \int_{\Y\times\Y} \mathbf{1}_{[y = y']} \, d\rho(y|x)d\rho(y'|x'), 
	\end{align}
 which is the probability that two observers $x$ and $x'$ are affiliated to the same class. One can see that $\eta$ only depends on the conditional distributions of $Y$ conditioned on the observers. For any $x,x',x''\in\X$,  if $\rho(\cdot|x') = \rho(\cdot|x'')$, then there holds $\eta(x,x') = \eta(x,x'')$. In addition, it can be readily observed that $\eta(x,x')$ is symmetric, \ie $\eta(x,x') = \eta(x',x)$ for any $x,x'\in\X$.

    Denote by $P_x := [Prob\{Y = y_1 | X = x\}, \ldots, Prob\{Y = y_m | X = x\}] \in \R^m$ the probability distribution vector of $Y$ conditioned on $X=x$. Then the conditional probability \eqref{eta} can be represented as the standard inner product in the Euclidean space:
    \begin{align*}
        \eta(x,x') &= \sum_{i=1}^mProb\{Y = Y' = y_i|X = x, X' = x'\}\\
        &= \sum_{i=1}^mProb\{Y = y_i|X = x\}Prob\{Y' = y_i| X' = x'\}\\
        &= \langle P_x, P_{x'}\rangle.
    \end{align*}

	We aim to estimate the excess generalization error $\E(\hat{d}_z) - \E(d_\rho) = \{\E(\hat{d}_z) - \E(d_\calH)\} + \{\E(d_\calH) - \E(d_\rho)\}$, where $\hat{d}_z$ is the minimizer of the empirical error over the hypothesis space. In this section, we focus on the hinge loss 
    $$\ell(\tau(y,y')d(x, x')) = \big(1 + \tau(y, y') d(x, x')\big)_+,$$
    where $(\cdot)_+ = \max\{0, \cdot\}$. 

    \subsection{Explicit form of the true metric} 
    Understanding the true metric $d_\rho$ is crucial to estimating the approximation error. The following theorem presents an explicit representation of the true metric $d_\rho$ for the hinge loss.
	\begin{theorem}\label{true_predictor}
		The true metric with the hinge loss can be represented as
        \begin{align*}
            d_\rho(x, x') &= sgn(1 - 2 \eta(x, x'))\\
            &= sgn(1 - 2 \langle P_x, P_{x'}\rangle)
        \end{align*}
		for almost every pair $x, x'\in\X$.
	\end{theorem}	
	\begin{proof}
 By the tower property of the conditional expectation, the generalization error with a metric $d$ can be written as
    \begin{align*}
        \E(d) = \bE_{X,X'}\left[\bE_{Y|X, Y'|X'}\left[ \ell(\tau(Y, Y') d(X, X') \right]\right].
    \end{align*}
   The above observation implies that for almost every pair $x, x' \in \X$, the true metric $d_\rho(x,x')$ is obtained by minimizing the inner conditional expectation which can be rewritten as
    \begin{align}\label{eq:d_rho}
        &d_\rho(x,x')\nonumber\\ & = \arg\min_{t \in \R} \bE_{Y|X = x, Y'|X' = x'}\left[ \ell(\tau(Y, Y') t)\right] \nonumber \\
        &= \arg\min_{t \in \R} Prob \{Y = Y'| X = x, X' = x'\}\ell (t) + Prob\{Y \neq Y'| X = x, X' = x'\} \ell(- t) \nonumber \\
        &= \arg\min_{t \in \R} \eta(x,x')\ell(t) + (1 - \eta(x,x'))\ell(- t),
    \end{align}
    where the second equality is due to the definition of $\tau(y, y')$.
    
	According to \eqref{eq:d_rho}, the true metric with the hinge loss can be expressed as
		\begin{align*}
			d_\rho(x,x') &= \arg\min_{t\in\R}\left\{ \eta(x, x')\left(1 + t\right)_+ + \left(1 - \eta(x, x')\right)\left(1 - t\right)_+\right\}\\
			&= \arg\min_{t\in\R}\left\{\begin{aligned}
				&(1 - \eta(x, x'))(1 - t), &&\text{if } t< - 1,\\
				&(2\eta(x, x') - 1)t + 1,&& \text{if } t\in [-1, 1],\\
				&\eta(x,x')(1 + t), && \text{if } t > 1.
			\end{aligned}
			\right.
		\end{align*}
	For the case $\eta(x, x') > 1/2$, note that the objective function is piecewise linear and tends to $+\infty$ as $t \to \pm \infty$, we can derive that $d_\rho(x, x') = - 1$ and the corresponding minimum of the objective function is $2(1 - \eta(x, x'))$. For the case $\eta(x, x') < 1/2$, we can get that  $d_\rho(x, x') = 1$ and the corresponding minimum is $2\eta(x, x')$. For the case $\eta(x,x') = 1/2$, we have $d_\rho(x,x')\in[-1,1]$ and the corresponding minimum be $0$. Then we can conclude that $d_\rho(x,x') = sgn(1 - 2\eta(x,x'))$, which completes the proof.
	\end{proof}
    \begin{remark}
        Unlike using the Mahalanobis distance to measure the similarity between sample pairs, deep metric learning \cite{deep2019huai,deepsurvey2019,revisiting2020} aims to learn a nonlinear embedding function $\phi: \X \to \Phi \subset \R^s$ with  $s\in\N^+$, such that similar data points $x, x'$ are close in the embedding space $\Phi$ under a predefined distance function $d(\phi(x), \phi(x'))$ and far from each other if they are dissimilar.  Note Theorem~\ref{true_predictor} shows that  for the hinge loss, the true predictor $d_\rho(x,x') = sgn(1 - 2\langle P_x, P_{x'}\rangle)$. Then the true embedding function is $\phi(x) = P_x$ under this setting, and the distance function can be further defined as $d(\phi(x), \phi(x')) = sgn(1 - 2\langle P(x), P(x')\rangle)$. This indicates that learning a nonlinear embedding function in deep metric learning with the hinge loss is in fact learning the conditional probability $P_x$.
    \end{remark}
 


    According to the specific structure of the true metric $d_\rho$ given in Theorem~\ref{true_predictor}, we can construct a structured deep network with ReLU activation as an approximation of $d_\rho$ and further design the corresponding hypothesis space.  Before that, we first introduce some assumptions and useful lemmas. 
    
    To define the Sobolev smoothness of the conditional probabilities $p_i$ for $i=1,\ldots,m$, we take $\X = [0,1]^p$ in the remainder of this section.
  We define the Sobolev space $W^{r,\infty}([0,1]^p)$ as the space of functions on $[0,1]^p$ lying in $L^\infty([0,1]^p)$ along with their partial derivatives up to order $r$. The norm in $W^{r,\infty}([0,1]^p)$ is defined as
    \begin{align*}
        \|f\|_{W^{r,\infty}([0,1]^p)}:= \max_{\alpha \in \mathbb{Z}^p_+: \|\alpha\|_1 \le r} \|D^\alpha f\|_{L^\infty([0,1]^p)},
    \end{align*}
    where $\|\alpha\|_1 = \sum_{i=1}^p |\alpha_i|$ denotes the $l^1$ norm of $\alpha=(\alpha_1,\ldots,\alpha_p)\in\mathbb{Z}^p_+$, and $D^\alpha f = \frac{\partial ^{\|\alpha\|_1} f}{\partial^{\alpha_1}x_1\cdots\partial^{\alpha_d}x_d}$ denotes the partial derivatives of $f$ with order $\alpha$.
 
   
    For notational simplicity, denote by $p_i(x)=Prob\{Y = y_i | X = x\}$ the $i$-th component of $P_x$. The following assumption means that all $r$-th partial derivatives of $p_i(x)$ exist and their $L^\infty$ norms are bounded by $1$.
    \begin{assumption}\label{smooth_regularity}
        The conditional probability $p_i(x) \in W^{r,\infty}([0,1]^p)$ has the Sobolev norm not greater than $1$ for each $i =1,\ldots,m$ (This upper bound for the Sobolev norm can be extended to any finite constant, for simplicity we set it to $1$).
    \end{assumption}
In binary classification problems, we often introduce a noise condition that says the ambiguous points $Prob\{Y = 1|X=x\} \approx 1/2$ occur with a small probability. When this condition is satisfied, it suggests that the classification problem is well-posed and learning algorithms have the potential to achieve faster convergence rates. Similarly, we can extend this notion in metric and similarity learning to suggest that the probability of $\eta(X, X') \approx 1/2$ is relatively small. In the rest of this paper, we denote by $C_{\alpha,\beta,\gamma}$ a constant only depending on parameters $\alpha,\beta,\gamma$, and it may differ from line to line.   
	\begin{assumption}[Tsybakov's noise condition]\label{noise_condition}
		There exist constants $\theta > 0$ and $C_\theta>0$ such that for any $t>0$,
		$$Prob\{|\eta(X, X') - 1/2| \le t\} \le C_\theta t^\theta. $$
	\end{assumption}
\subsection{Structured deep ReLU networks} 

For \(L\in\mathbb{N}\), let \(h:\mathbb{R}^p\to\mathbb{R}\) be a deep ReLU network of depth \(L\) with activation function \(\sigma(x)=\max\{x,0\}\), defined by
\begin{align}\label{eq:nns}
        h(x) = \sigma(T_L \sigma(T_{L-1}\cdots \sigma(T_1x + b_1)\cdots + b_{L-1}) + b_L),
    \end{align}
where \(\sigma\) acts componentwise. For each \(l=1,\ldots,L\), \(T_l\in\mathbb{R}^{w_l\times w_{l-1}}\) and \(b_l\in\mathbb{R}^{w_l}\) denote, respectively, the weight matrix and bias vector of the \(l\)-th layer, where \(w_l\in\mathbb{N}\) is the width of the \(l\)-th layer and \(w_0=p\) is the input dimension. We define the number of nonzero parameters and the number of computational units of \(h\) by $h$ be $\sum_{l=1}^L \|T_l\|_0 + \|b_l\|_0$ and $\sum_{l=1}^L w_l$  respectively, where \(\|\cdot\|_0\) denotes the number of nonzero entries of the corresponding matrix or vector.
    
    The following two lemmas established in \cite{yarotsky} show the expressive power of deep ReLU networks in the setting of approximations in Sobolev spaces and the product function. 
    \begin{lemma}\label{approximate_smooth}
        For any $p, r \in \N$, $\epsilon \in (0,1/2)$ and any function $f \in W^{r,\infty}([0,1]^p)$ with Sobolev norm not larger than $1$, there exists a deep ReLU network $h$ with depth at most $C_{p,r} \log(1/\epsilon)$ and the number of nonzero weights and computational units at most $C_{p,r} \epsilon^{-\frac{p}{r}} \log(1/\epsilon)$ such that
        \begin{align*}
            \|h - f\|_{L^\infty([0,1]^p)} \le \epsilon.
        \end{align*}
    \end{lemma}
     We will use a deep ReLU network $h_i$ to approximate the conditional probability $p_i\in [0,1]$ later. According to Lemma \ref{approximate_smooth}, if Assumption \ref{smooth_regularity} holds, we have $h_i \in [-1,2]$.
	\begin{lemma}\label{approximate_multi}
		For any $\epsilon \in (0, 1/2)$, there exists a deep ReLU network $\phi$ with the depth and the number of weights and computation units at most $C\log(1/\epsilon)$ such that
		\begin{align*}
		     \|\phi(x,y) - xy\|_{L^\infty([-1,2]^2)} \le \epsilon \ \text{ and } \ \phi(x, y) = 0  \ \text{ if } \ xy = 0 .
		\end{align*} 
	\end{lemma}

We are now ready to construct a structured deep ReLU network to approximate the true metric \(d_\rho\), defined by
\begin{equation}\label{eq:structured-NNs}
    d(x,x'):= F_a\Big(1 - 2\sum_{i = 1}^m \phi(h_i(x), h_i(x'))\Big),
\end{equation}
where  $F_a(t): = \frac{1}{a} \big(\sigma(t + a) - \sigma(t - a) -a\big)$ is a fixed two-layer ReLU network with \(a>0\). The value of \(a\) will be chosen later, see Theorems~\ref{approximation_error} and \ref{thm:excess} for details. Here, \(\phi\) is the deep ReLU network introduced in Lemma~\ref{approximate_multi} for approximating the product function \(xy\), and each \(h_i\) is a sub-network of the form \eqref{eq:nns}. We assume that all sub-networks \(h_i\), \(i=1,\ldots,m\), have the same depth. 

\smallskip

\noindent\textbf{Main idea of the construction.}
The main idea behind the construction of \eqref{eq:structured-NNs} is to approximate the true metric \(d_\rho\) component by component according to its representation $d_\rho(x, x')  =sgn(1 - 2 \sum_{i=1}^m p_i(x)p_i(x'))$ given in Theorem~\ref{true_predictor}. Specifically, we first construct \(m\) sub-networks \(h_i\), \(i=1,\ldots,m\), to approximate the conditional probabilities \(p_i\), respectively. Next, we introduce a deep ReLU network \(\phi(h_i(x),h_i(x'))\) to approximate the product \(h_i(x)h_i(x')\). By Lemma~\ref{approximate_multi}, \(\phi(h_i(x),h_i(x'))\) provides a reliable approximation of \(h_i(x)h_i(x')\). Finally, the output layer \(F_a(\cdot)\) is used to approximate the sign function, following the idea in \cite{zhou2024classification}. Indeed, \(F_a(\cdot)\) coincides with \(sgn(\cdot)\) on the interval \((-\infty,-a)\cup(a,\infty)\), while being linear on \([-a,a]\) so as to approximate the discontinuity of \(sgn(\cdot)\) near the origin. The resulting structure of \eqref{eq:structured-NNs} is illustrated in Figure~\ref{graph}.
\begin{figure}[h]
    \centering
    \includegraphics[scale=0.4]{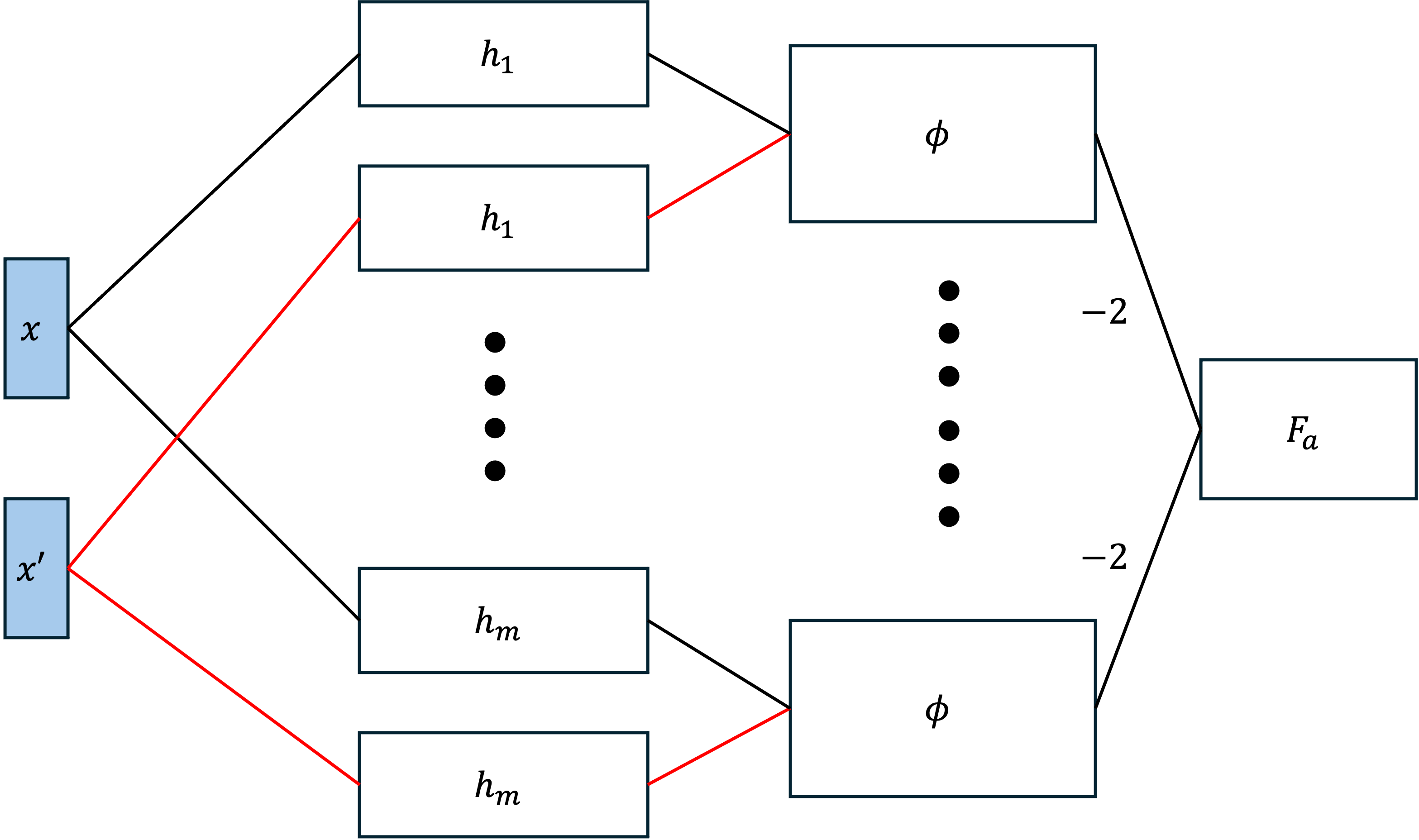}
    \caption{Structure of the proposed deep ReLU network \eqref{eq:structured-NNs} with inputs \(x,x'\in\X\).}
    \label{graph}
\end{figure}

\subsection{Error analysis}   In this subsection, we estimate the approximation error and the estimation error, and then derive the excess generalization error bound for the constructed deep ReLU networks.
We say that a network \(d\) has complexity \((L,W,U)\) if its depth, the number of possibly nonzero weights, and the number of computational units are \(L\), \(W\), and \(U\), respectively. The hypothesis space consisting of structured deep networks of the form \eqref{eq:structured-NNs} is defined by
\begin{align}\label{hypothesis_space}
        \calH = \calH(L,W,U) = \big\{d(x,x') \text{ of the form \eqref{eq:structured-NNs}}&:\text{the complexity of $d$ }\\
        &\text{ does not exceed $(L,W,U)$}\big\}.\nonumber
    \end{align}
Here, the complexity of \(d\) is obtained by summing the corresponding quantities over its sub-networks. Specifically, the depth of \(d\) is given by the sum of the depths of \(h_1\), \(\phi\), and \(F_a\) (recall that all \(h_i\) have the same depth). Moreover, the numbers of possibly nonzero weights and computational units of \(d\) are
\[
2\sum_{i=1}^m W_{h_i} + mW_\phi + W_{F_a} + c
\quad\text{and}\quad
2\sum_{i=1}^m U_{h_i} + mU_\phi + U_{F_a} + c,
\]
respectively, where \(W_\gamma\) and \(U_\gamma\) denote the corresponding quantities of the sub-network \(\gamma \in \{h_1,\ldots,h_m,\phi,F_a\}\), and \(c\) is an absolute constant. The capacity of the hypothesis space \(\calH\) is measured by \((L,W,U)\). As these parameters increase, the capacity of \(\calH\) becomes larger.

\smallskip
    
    The following theorem establishes approximation error bounds of the structured deep networks of the form \eqref{eq:structured-NNs}.     
    \begin{theorem}[Approximation error]\label{approximation_error}
		Suppose Assumptions \ref{smooth_regularity} and \ref{noise_condition} hold. For any $\epsilon \in(0, 1/2)$, if the hypothesis space $\calH$ defined in \eqref{hypothesis_space} has the depth $C_{p,r,m, \theta} \log(1/\epsilon)$ and the number of possibly nonzero weights and computation units $C_{p,r,m, \theta}\epsilon^{-\frac{p}{r(\theta + 1)}} \log(1/\epsilon)$, then there exists a deep ReLU network $d_\calH\in \calH$ of the form \eqref{eq:structured-NNs} with $a = C_\theta \epsilon^{\frac{1}{\theta+1}}$ such that
		$$\E(d_\calH) - \E(d_\rho) \le \epsilon.$$

	\end{theorem}
	
	\begin{proof}
		Note $|\tau(y,y')| = 1$ and $|sgn(t)| \le 1$ for any $y,y'\in\Y$ and $t\in\R$. For any metric $d:[0,1]^p\times[0,1]^p\to\Y$ with $\|d\|_{L^\infty([0,1]^{2p})} \le 1$, the conditional excess risk can be written as
        \begin{align*}
            &\bE_{Y|X=x, Y'|X'=x'}\left[\ell(\tau(Y,Y')d(x,x')) - \ell(\tau(Y,Y')d_\rho(x,x'))\right]\\
            =\ & \bE_{Y|X=x, Y|X'=x'}\left[\big(1 + \tau(Y,Y')d(x,x')\big)_+ - \big(1 + \tau(Y,Y')sgn(1 - 2\eta(x,x'))\big)_+\right]\\
            =\ & \bE_{Y|X=x, Y|X'=x'}\left[\big(1 + \tau(Y,Y')d(x,x')\big) - \big(1 + \tau(Y,Y')sgn(1 - 2\eta(x,x'))\big)\right]\\
            =\ & \bE_{Y|X=x, Y|X'=x'}\left[\tau(Y,Y')\big(d(x,x') - sgn(1-2\eta(x,x'))\big)\right]\\
            =\ & (2\eta(x,x') - 1)(d(x,x') - sgn(1 - 2\eta(x,x'))\\
            =\ & |2\eta(x,x') - 1||d(x,x') - sgn(1 - 2\eta(x,x')|,
        \end{align*}
        where in the first equality we have used Theorem~\ref{true_predictor},  the second equality follows from the fact that the terms inside $(\cdot)_+$ are nonnegative,  and the last equality is obtained by discussing the sign of $2\eta - 1$ and the fact $\|d\|_{L^\infty([0,1]^{2p})} \le 1$.
        
        Therefore, for $\delta \in(0,1)$ and $d \in \calH$, the excess error
		\begin{align}\label{eq:approx-1}
			&\E(d) - \E(d_\rho)\nonumber\\ &= \bE_{X,X'}\left[\bE_{Y|X,Y'|X'}\left[ \ell(\tau(Y,Y')d(X,X')) - \ell(\tau(Y,Y')d_\rho(X,X'))\right]\right]\nonumber\\
            &= \bE_{X,X'}\left[|2\eta(X, X') - 1| \left|d(X, X') - sgn(1 - 2\eta(X, X'))\right|\right]\nonumber\\
			&=  \int_{|1 - 2\eta(x, x')| > 2\delta}|2\eta(x, x') - 1| \left|d(x, x') - sgn(1 - 2\eta(x, x'))\right| d\rho_X(x) d\rho_X(x')\nonumber \\
			&\quad  + \int_{|1 - 2\eta(x, x')| \le 2\delta}|2\eta(x, x') - 1| \left|d(x, x') - sgn(1 - 2\eta(x, x'))\right| d\rho_X(x) d\rho_X(x')\nonumber\\
            &\le \int_{|1 - 2\eta(x, x')| > 2\delta}|2\eta(x, x') - 1| \left|d(x, x') - sgn(1 - 2\eta(x, x'))\right| d\rho_X(x) d\rho_X(x')\nonumber \\
			&\quad + 4C_\theta \delta^{\theta+1},
		\end{align}
   where in the last step we have used Assumption~\ref{noise_condition} and the condition $|\eta-1/2| \le \delta$. 

  Let us analyze the first term in \eqref{eq:approx-1}. 
	Note that Lemma \ref{approximate_smooth} with $f=p_i$ implies that for each $p_i$ with approximation accuracy $\frac{\delta}{8m}$, there exists a ReLU network $h_i$ with depth at most $C_{p,r,m} \log(1/\delta)$ and the number of possibly nonzero weights and computation units at most $C_{p,r,m}\delta^{-\frac{p}{r}} \log(1/\delta)$. Since $p_i(x) \in [0,1]$ and $|h_i(x) - p_i(x)| \le \frac{\delta}{8m} < 1$ for any $x\in[0,1]^p$ and $i=1,\ldots,m$, we know $h_i(x) \in [-1, 2]$ for any $x\in[0,1]^p$ and $i=1,\ldots,m$.
 
 Let $\phi$ be the ReLU network in Lemma \ref{approximate_multi} with approximation accuracy $\frac{\delta}{8m}$, and $F_a(x) = \frac{1}{\delta}(\sigma(x+\delta) - \sigma(x - \delta) - \delta)$ with $a = \delta$. Consider constructing $d_\calH$ using the above sub-networks, \ie
    \begin{align*}
        d_\calH(x,x') = F_a\Big(1 - 2\sum_{i=1}^m \phi(h_i(x), h_i(x'))\Big).
    \end{align*}
    We know $d_\calH$ is a deep ReLU network from $[0,1]^p\times[0,1]^p$ to $[-1,1]$ with depth at most $C_{p,r,m} \log(1/\delta)$ and the number of possibly nonzero weights and computation units at most $C_{p,r,m}\delta^{-\frac{p}{r}} \log(1/\delta)$. 
    
    We claim that the sign of the term $1 - 2\sum_{i = 1}^m \phi(h_i(x), h_i(x'))$ inside $F_a$ coincides with $1 - 2\eta(x,x')$ if $|1 - 2\eta(x,x')| > 2\delta$, which can be proved by showing that $|(1 - 2\sum_{i = 1}^m \phi(h_i(x), h_i(x'))$ $ - (1- 2\eta(x,x'))| < \delta$. Because when $1 - 2\eta(x, x') > 2\delta$ or $1 - 2\eta(x, x') < -2\delta$, we must have $1 - 2\sum_{i = 1}^m \phi(h_i(x), h_i(x')) > \delta$ or $1 - 2\sum_{i = 1}^m \phi(h_i(x), h_i(x')) < -\delta$. Indeed, for any $x, x' \in [0,1]^p$, there holds
	\begin{align*}
		&\Big|\Big(1 - 2\sum_{i = 1}^m \phi(h_i(x), h_i(x'))\Big) - \left(1 - 2\eta(x,x')\right)\Big| \\
		&\le 2\sum_{i = 1}^m \Big|\phi(h_i(x), h_i(x')) - p_i(x)p_i(x')\Big| \\
		&\le 2\sum_{i = 1}^m \Big|\phi(h_i(x), h_i(x')) - h_i(x)h_i(x') \Big| + \Big|h_i(x)h_i(x') -  p_i(x)p_i(x')\Big| \\
		&\le 2\sum_{i = 1}^m \Big(\frac{\delta}{8m} + |h_i(x)| |h_i(x') - p_i(x')| + |p_i(x)| |h_i(x) - p_i(x)| \Big)\\
		&\le  2\sum_{i = 1}^m \Big(\frac{\delta}{8m} + \frac{\delta}{4m} + \frac{\delta}{8m} \Big)= \delta. 
	\end{align*}
    According to the definition of $F_a$ and recall that $a = \delta$, we know $F_a(t) = 1$ if $t>\delta$ and  $F_a(t) = -1$ if $t<-\delta$. Therefore, we conclude that $d_\calH(x, x') - sgn(1 - 2\eta(x, x')) = 0$ if $|1 - 2\eta(x,x')| > 2\delta$. Thus, the first term in \eqref{eq:approx-1} vanishes.
 
 Taking $d = d_\calH$ in the above approximation error bound, we last have
	\begin{align*}
			\E(d_\calH) - \E(d_\rho) \le 4C_\theta \delta^{\theta + 1}.
	\end{align*}
	By setting $\epsilon = 4C_\theta\delta^{\theta+1}$ and recall $a = \delta$, we get the desired results.
	\end{proof}

 \begin{remark}
 Theorem~\ref{approximation_error} implies that the approximation ability of \(d_{\calH}\) improves as the capacity of the hypothesis space \(\calH\) increases. However, as the hypothesis space becomes larger, the model may become overly flexible, leading to a larger estimation error. This indicates a trade-off between the approximation error and the estimation error. We will choose a proper hypothesis space to obtain the explicit excess generalization error rate (see Theorem~\ref{thm:excess}). 
 \end{remark}

 To derive upper bounds for the estimation error, we measure the capacity of the hypothesis space by its pseudo-dimension. The pseudo-dimension is defined as the VC-dimension of the corresponding subgraph class. A comprehensive definition of VC-dimension can be found in \cite{gyorfi,HDS,zhangTong_ml}. 
    \begin{definition}
        Let $\F$ be a class of functions from $[0,1]^p$ to $\R$,  and $\F^+:= \{(x,t)\in[0,1]^p\times\R: f(x) > t, f \in \F\}$ be the corresponding sub-graph set. The pseudo-dimension $Pdim(\F)$ of $\F$ is defined as
		\begin{center}
			$Pdim(\F):= VC(\F^+)$,
		\end{center}
		where $VC(\F^+)$ is the VC-dimension of $\F^+$. Furthermore, if $Pdim(\F) < \infty$, then we call $\F$ a VC-class.
    \end{definition}
    If we solely apply the uniform boundedness (first-order condition) of the hypothesis space, the estimation error bound is of the order $O(1/\sqrt{n})$ \cite{cao, ranking}, which is often quite loose. To derive a tighter upper bound of the estimation error, the variance condition (second-order condition) should be taken into consideration.
    \begin{definition}\label{variance-expectation}
    Let $\beta\in(0,1]$ and $M>0$, and $\F \subset L^2(\X\times\X, \rho\times\rho)$ is a function class with nonnegative first order moment, \ie for any $f\in\F$, $\bE[f] \ge 0$.
 We say $\F$ has a variance-expectation bound with parameter pair $(\beta, M)$, if for any $f \in \F$,
		$$\mathbb{E}[f^2] \le M(\mathbb{E}[f])^\beta.$$
	\end{definition}
    \begin{definition}
    Let $\Omega \subset \R^p$. 
        We say a function $g:\Omega \to\R$ is Lipschitz continuous with Lipschitz constant $\alpha>0$ if for any $x, y \in \Omega$,
        \begin{align*}
            |g(x) - g(y)| \le \alpha \|x - y\|_2,
        \end{align*}
        where $\|\cdot\|_2$ is the Euclidean norm.
    \end{definition}

    The following lemma is a specialization of part~(a) of Theorem~1 in \cite{zhou2023fine} to our metric learning setting. 
    Indeed, the symmetry of the loss function (III.2) in \cite{zhou2023fine} is satisfied in our setting, and Assumption~3 in \cite{zhou2023fine} follows from the standard entropy bound for a uniformly bounded hypothesis class of finite pseudo-dimension (see \cite{VC} for more details). 
By applying Theorem~1(a) in \cite{zhou2023fine} with these observations and absorbing the resulting constants and logarithmic factors into the generic constant, we obtain the following estimate.
   \begin{lemma}\label{estimation_error}
		 Let $V= Pdim(\calH)$ be the pseudo-dimension of the hypothesis space $\calH$, and $\ell(\tau(y,y')d(x,x'))$ be the loss function for $(z,z')\in\Z\times\Z$, where $d$ is a predictor from $\X\times\X$ to $\R$. Suppose the following conditions hold for some $\alpha > 0$, $s>0$, $\beta\in(0,1]$ and $M > 0$.
        \begin{itemize}
            \item for any $y,y' \in \Y$, the loss function $\ell(\tau(y,y') \cdot)$ is Lipschitz continuous with Lipschitz constant $\alpha$,  
            \item for any $d\in\calH\cup\{d_\rho\}$ and for almost every sample pair $(z,z')\in\Z\times\Z$, there holds $\ell(\tau(y,y')d(x,x')) = \ell(\tau(y',y)d(x',x))$ and $\|d\|_{L^\infty(\X\times\X)} \le s \in \R^+$,
            \item  the shifted hypothesis space $\{\ell(\tau(y,y')d(x,x')) - \ell(\tau(y,y')d_\rho(x,x')): d \in \calH\}$ has a variance-expectation bound with parameter pair $(\beta, M)$,
        \end{itemize}
       then for any $\delta \in (0, 1/2)$, with probability at least $1 - \delta$, there holds
		\begin{align*}
			\E(\hat{d}_z) - \E(d_\rho) \le  C_{s,\alpha,M,\beta} \Big(\dfrac{V\log(n)\log^2(4/\delta)}{n}\Big)^{\frac{1}{2-\beta}} + 2(1+\beta)\Big(\E(d_\calH) - \E(d_\rho)\Big).
		\end{align*}
	\end{lemma}

    The first two conditions in this lemma are easy to verify. Indeed, the Lipschitz continuity of the loss $\ell$ with Lipschitz constant $\alpha=1$ simply follows from the Lipschitz continuity of the hinge loss.  It is evident that the hypothesis functions and the true metric are uniformly bounded by $s = 1$, and the symmetry of the loss $\ell$ is obtained by the symmetries of the reducing function $\tau$, the hypothesis functions $d$ and the true metric $d_\rho$ with respect to $y,y'$ and $x,x'$ respectively.  All that remains is to check whether the shifted hypothesis space has a variance-expectation bound. The following proposition shows that this bound can be established under Tsybakov's noise condition.

	\begin{proposition}[Variance-expectation bound]\label{variance-expectation_prop}
		Suppose Assumption \ref{noise_condition} holds. For any $d \in \calH$, the shifted hypothesis space $\{\ell(\tau(y,y')d(x,x')) - \ell(\tau(y,y')d_\rho(x,x')):d \in \calH\}$ has a variance-expectation bound with parameter pair $\Big(\frac{\theta}{\theta + 1}, 2^{\frac{3}{\theta+1}} C_\theta^{\frac{1}{\theta + 1}}\Big)$. 
	\end{proposition}
	\begin{proof}
		As in the proof of Theorem \ref{approximation_error}, for any $d \in \calH$, we can show
		\begin{align*}
			\E(d) - \E(d_\rho) &= \bE_{X,X'}\left[|2\eta(X, X') - 1| \left|d(X, X') - sgn(1 - 2\eta(X, X'))\right|\right]\\
			&= \int_{\X\times\X}|2\eta(x, x') - 1| \left|d(x, x') - sgn(1 - 2\eta(x, x'))\right| d\rho_X(x) d\rho_X(x').
		\end{align*}
	    Let $q_d(z,z'):= \ell(\tau(y,y')d(x,x')) - \ell(\tau(y,y')d_\rho(x,x'))$. Note that $\tau(y,y')d(x,x') \in [-1,1]$ for any $z,z'\in\Z$ and $d\in\calH\cup\{d_\rho\}$. Then there holds $\ell(\tau(y,y')d(x,x')) = 1 + \tau(y,y')d(x,x')$. Since $|\tau(y,y')|^2 = 1$, we have
		\begin{align*}
			&\bE_{Z,Z'}[q_d^2(Z,Z')]\\
   &= \bE_{X,X'}[|d(X,X') - d_\rho(X,X')|^2]\\
            &= \int_{\X\times\X} \left|d(x, x') - sgn(1 - 2\eta(x, x'))\right|^2 d\rho_X(x) d\rho_X(x')\\
			&= \int_{|2\eta(x,x') - 1| > t} \left|d(x, x') - sgn(1 - 2\eta(x, x'))\right|^2 d\rho_X(x) d\rho_X(x')\\
			&\ \ \ + \int_{|2\eta(x,x') - 1| \le t} \left|d(x, x') - sgn(1 - 2\eta(x, x'))\right|^2 d\rho_X(x) d\rho_X(x')\\
			&\le \int_{|2\eta(x,x') - 1| > t} \left|d(x, x') - sgn(1 - 2\eta(x, x'))\right| \frac{|2\eta(x,x') - 1|}{t} d\rho_X(x) d\rho_X(x')\\
			&\ \ \ + 4Prob\{|\eta(X, X') - 1/2| \le t/2\}\\
			&\le \frac{\E(d) - \E(d_\rho)}{t} + 4C_\theta \big(\frac{t}{2}\big)^\theta,
		\end{align*}
	where $t>0$, and in the last inequality we have used Tsybakov's noise condition directly. 
 
 By choosing $t = \Big(\frac{\E(d) - \E(d_\rho)}{2^{2-\theta}C_\theta}\Big)^{\frac{1}{\theta + 1}}$, we obtain
 $$\bE_{Z,Z'}[q_d^2(Z,Z')] \le 2^{\frac{3}{\theta +1}} C_\theta^{\frac{1}{\theta + 1}} \Big(\E(d) - \E(d_\rho)\Big)^{\frac{\theta}{\theta + 1}},$$
 which completes the proof. 
	\end{proof}\jd

    The estimation error is estimated in the following theorem by combining Lemma~\ref{estimation_error} and Proposition~\ref{variance-expectation_prop} together.
    \begin{theorem}[Estimation error]\label{thm:estimationerror}
    Let $V= Pdim(\calH)$ be the pseudo-dimension of the hypothesis space $\calH$ defined in \eqref{hypothesis_space}. 
    For any $\delta\in(0,1/2)$, with probability at least $1-\delta$, there holds
    \begin{align*}
			\E(\hat{d}_z) - \E(d_\rho) \le  C_{\theta} \Big(\dfrac{V\log(n)\log^2(4/\delta)}{n}\Big)^{\frac{\theta + 1}{\theta + 2}} + \frac{4\theta + 2}{\theta + 1}\Big(\E(d_\calH) - \E(d_\rho)\Big).
		\end{align*}
    \end{theorem}

 Now, we can obtain an excess generalization error bound by combining the approximation error bound (Theorem~\ref{approximation_error}) and the estimation error bound (Theorem~\ref{thm:estimationerror}). 
 An explicit excess generalization error rate can be further established by carefully trading off the estimation error and the approximation error. 
    
	\begin{theorem}[Excess generalization error]\label{thm:excess}
           Suppose Assumptions \ref{smooth_regularity} and \ref{noise_condition} hold and let $L \in \N$. Consider the hypothesis space $\calH = \calH(L,W,U)$ defined in \eqref{hypothesis_space} with $W = U = \lceil C_{p,r,m,\theta} \exp\{L\}\rceil$. For any $\delta \in (0,1/2)$, with probability at least $1 - \delta$, there holds
		$$\E(\hat{d}_z) - \E(d_\rho) \le C_{p,r,m,\theta} \log^2(4/\delta)\bigg\{\Big(\frac{L^2\exp\{L\}\log n}{n}\Big)^{\frac{\theta + 1}{\theta + 2}}  + \Big(\frac{L}{\exp\{L\}}\Big)^{\frac{(\theta + 1)r}{p}} \bigg\}.$$
		By setting $L = \left\lceil\frac{p}{p + (\theta + 2) r}\log\Big(\dfrac{n}{\log n}\Big)\right\rceil$, we have
		$$\E(\hat{d}_z) - \E(d_\rho) \le C_{p,r,m,\theta} \log^2(4/\delta) \log^4(n) n^{-\frac{(\theta + 1) r}{p + (\theta + 2)r}}.$$
	\end{theorem}
	\begin{proof}
            According to Theorem 7 in \cite{VC}, we know $V \le CLW\log U$, where $W$ and $U$ are the number of possibly nonzero weights and computation units, respectively. By setting $W = U = \lceil C_{p,r,m,\theta}\exp\{L\}\rceil$ and applying Theorem \ref{approximation_error} with $\epsilon = C_{p,r,m,\theta}(\frac{L}{\exp\{L\}})^{\frac{r(\theta+1)}{p}}$ and $a = C_{p,r,m,\theta}(\frac{L}{\exp\{L\}})^{\frac{r}{p}}$, the approximation error can be estimated as \[\E(d_\calH) - \E(d_\rho) \le C_{p,r,m,\theta} \Big(\frac{L}{\exp\{L\}}\Big)^{\frac{(\theta + 1)r}{p}}.\]  
            According to Theorem \ref{thm:estimationerror} and applying the above two error bounds of the pseudo-dimension and the approximation error, we have, with probability at least $1 - \delta$,
            $$\E(\hat{d}_z) - \E(d_\rho) \le C_{p,r,m,\theta} \log^2(4/\delta)\bigg\{\Big(\frac{L^2\exp\{L\}\log n}{n}\Big)^{\frac{\theta + 1}{\theta + 2}}  + \Big(\frac{L}{\exp\{L\}}\Big)^{\frac{(\theta + 1)r}{p}} \bigg\}.$$
            The proof is completed by setting $L = \left\lceil\frac{p}{p + (\theta + 2) r}\log\Big(\dfrac{n}{\log n}\Big)\right\rceil$. 
	\end{proof}
    Theorem~\ref{thm:excess} shows that the excess generalization error bound is of order (up to a logrithmic term) $O(n^{-\frac{(\theta + 1) r}{p + (\theta + 2)r}})$. This rate is closely related to the dimension of the input space, the parameter $\theta$ in the noise condition, and the smoothness $r$ of the conditional probabilities. When the distribution $\rho$ has very low noise and the corresponding conditional probabilities are rather smooth, \ie parameters $\theta$ and $r$ are very large, then the learning rate can be of order $O(n^{-1+\epsilon})$ with a small $\epsilon > 0$.
\begin{remark}
    It is worth emphasizing that previous works \cite{cao,guo2014guaranteed,deep2019huai,jin2009,fast2019} on the study of generalization analysis for metric and similarity learning only derived the estimation error bounds. For instance, \cite{cao} established the upper bound for the estimation error of the order $O(\frac{1}{\sqrt{n}})$ for metric learning with the hinge loss, where they focused on learning the Mahalanobis distance. \cite{fast2019} showed that the convergence rate of the estimation error can achieve $O(\frac{1}{n})$ for metric learning with the smooth loss function and strongly convex objective. \cite{jin2009} established estimation error bounds of order $O(\frac{1}{\sqrt{n}})$ via the algorithmic stability for metric learning.
\end{remark}
\begin{remark}
   The assumptions used in the above theorem are tailored to the present proof technique in this paper.
    In particular, the proof of Theorem~\ref{approximation_error} relies on a uniform approximation of the conditional probability function $\eta(x,x')$ together with the outer approximation $F_a$ of the sign function.
    An $L^2$-type regularity assumption considered in \cite{xiang2009classification} may be more natural under additional regularity of the marginal distribution of $X$, for example, when $\rho_X$ has a uniformly bounded density with respect to the Lebesgue measure.
    However, such a condition is not immediate within our framework, since the present argument relies on uniform approximation in order to preserve the sign away from the decision boundary.
    Similarly, alternative low-noise assumptions, such as geometric noise conditions introduced in \cite{steinwart2007fast}, would require a separate formulation on $\mathcal X\times\mathcal X$ and a substantially different analysis. We leave these questions for future work. 
\end{remark} 

%% file: Regularites.tex
\section{Structural properties of the true metric under general losses}\label{ssec:regularity}
In this section, we investigate structural properties of the problem setting and the true metric for metric and similarity learning under a general loss function.
Specifically, we first show that the bias term $b$ is not intrinsic to the theoretical characterization of the true metric $d_\rho(x,x')$, and to assume that the output space \(\mathcal{Y}\) consists of finitely many labels, as we do throughout the paper. We then prove that the true metric is symmetric for almost all \(x,x'\in\X\), which provides further justification for the use of symmetric models such as the Mahalanobis distance in metric learning. Finally, we show that, for convex, nonnegative, and nondecreasing losses, the true metric between two identical samples is always less than or equal to that between two different samples.

\smallskip

\noindent \textbf{Bias term at the population level.}
Unlike many existing works \cite{cao,deep2019huai,jin2009,lei2016generalization,revisiting2020,fast2019}, we do not introduce a bias term \(b>0\) in the loss function \(\ell\). We show that, at the population level, this term only induces a translation of the true metric, and is therefore not essential for its theoretical characterization.

Denote by \[\E_b(d) = \bE_{Z,Z'}\left[\ell(\tau(Y,Y')(d(X,X') - b)\right]\] the generalization error of \(d\) with bias term \(b>0\). Let $\tilde{d}_\rho := \arg\min_{d\in\F}\E_b(d)$ be the corresponding true metric under this biased loss. Analogously to \eqref{eq:d_rho}, we have
\begin{align*}
    \tilde{d}_\rho(x,x')
    &= \arg\min_{t \in \R} \eta(x,x')\ell(t - b) + (1 - \eta(x,x'))\ell(b - t)\\
    &= b + \arg\min_{s \in \R} \eta(x,x')\ell(s) + (1 - \eta(x,x'))\ell(-s)\\
    &= b + d_\rho(x,x')
\end{align*}
for almost every pair \(x,x'\in\X\). This shows that the bias term \(b\) merely induces a constant shift and is therefore inessential in the theoretical characterization. Hence, throughout the paper, we set \(b=0\).

\smallskip

\noindent \textbf{The output space is finite.}
In almost all theoretical and empirical works on metric and similarity learning \cite{bar2005learning,cao,davis2007information,guo2014guaranteed,jin2009,kar2011similarity,lei2016generalization,fast2019}, the label space \(\Y\) is assumed to be finite, and often even binary, \ie \(\Y=\{+1,-1\}\). A finite label space corresponds to a discrete distribution of \(Y\). This naturally raises the question of what happens when the distribution of \(Y\) is continuous. The following proposition answers this question.

\begin{proposition}\label{discrete_distr}
If the distribution of \(Y\) is continuous, then the true metric is almost surely equal to a generalized constant \(c\in[-\infty,+\infty]\).
\end{proposition}

\begin{proof}
For almost every pair \(x,x'\in\X\), it follows from \eqref{eq:d_rho} that
\begin{align}\label{true_metric_opt}
    d_\rho(x,x')=\arg\min_{t\in\R}\eta(x,x')\ell(t)+(1-\eta(x,x'))\ell(-t).
\end{align}
Since the distribution of \(Y\) is continuous, we have \(Prob\{Y=Y'\}=0\). Hence, \(\eta(x,x')=0\) for almost every pair \(x,x'\in\X\). Combining this with \eqref{true_metric_opt}, we obtain
\begin{align*}
    d_\rho(x,x')=\arg\min_{t\in\R}\ell(-t).
\end{align*}
Therefore, \(d_\rho(x,x')\) is almost surely equal to a generalized constant \(c\in[-\infty,+\infty]\). This completes the proof.
\end{proof}

Proposition~\ref{discrete_distr} shows that, when the distribution of \(Y\) is continuous, one has \(d_\rho(x,x')=c\in[-\infty,+\infty]\) for almost every pair \(x,x'\in\X\). In other words, the distances or similarities between almost all sample pairs are identical. Therefore, it is natural to focus on the case where the label space \(\Y\) is finite.

\begin{remark}
The above proposition suggests that, for a general distribution of \(Y\), the continuous part plays no essential role in determining the true metric, since the conditional probability \(\eta\) depends only on the discrete part of the distribution. Moreover, this phenomenon is also tied to the choice of the reducing function \(\tau\). Indeed, we assume that \(\tau(y,y')=1\) only when \(y=y'\), an event that has probability zero when \(Y\) is continuously distributed. Under this setting, the labels therefore play essentially no role in the learning problem.
\end{remark}

     \smallskip

\noindent \textbf{Regularities of the true metric.}
In contrast to predictors in ranking problems \cite{agarwal,ranking,huang2023generalization}, a metric \(d\) is intended to reflect the distance or similarity between a given pair of objects, rather than a rank or an ordering. Therefore, it is natural to require that \(d\) be independent of the order of the sample pair \((x,x')\), that is, \(d(x,x')=d(x',x)\). The following proposition shows that the true metric \(d_\rho\) indeed satisfies this symmetry property. This provides theoretical justification for constructing hypothesis spaces using symmetric models such as the Mahalanobis distance \((x-x')^\top M(x-x')\) or the pairwise similarity function \(x^\top Mx'\).
	\begin{proposition}
		The true metric $d_\rho$ is symmetric, \ie almost surely we have
		$$d_\rho(x, x') = d_\rho(x', x).$$
	\end{proposition}
	\begin{proof}
		By \eqref{eq:d_rho}, the proof directly follows from the symmetry of $\eta(x,x')$.
	\end{proof}

In mathematical terms, the distance between any two identical points is defined to be zero. In particular, the distance $d(\phi(x),\phi(x'))=\|\phi(x)-\phi(x')\|_2$ used in deep metric learning \cite{deep2019huai,deepsurvey2019,revisiting2020} and the Mahalanobis distance $d(x,x')=(x-x')^\top M(x-x')$ used in traditional distance metric learning \cite{cao,fast2019} both vanish when \(x=x'\). Consequently, for all \(x,x'\in\X\), one has \(d(x,x)\le d(x,x')\). However, this property may not hold for the true metric.

For example, consider the true metric associated with the hinge loss; see Theorem~\ref{true_predictor}. Let the number of labels be \(m=3\), and suppose that $P_x= [3/5, 1/5, 1/5]$ and $P_{x'} = [1, 0, 0]$. 
Then $\eta(x,x) = \langle P_x, P_x\rangle = 11/25 < 3/5 = \langle P_x, P_{x'}\rangle = \eta(x,x')$ 
and hence
\[
d_\rho(x,x)=sgn(1-2\eta(x,x))=1
> -1
= sgn(1-2\eta(x,x'))
= d_\rho(x,x').
\]
This shows that in general, the true metric need not satisfy the intuitive property $d_\rho(x,x)\le d_\rho(x,x') $ for all $ x,x'\in\X.$
Moreover, this behavior is determined by the underlying distribution $\rho$. It is therefore natural to seek sufficient conditions on $\rho$ under which $d_\rho(x,x)\le d_\rho(x,x') \text{ for all } x,x'\in\X.$

Note the minimizer in \eqref{eq:d_rho} may not be unique, we first introduce the following definition. 
    \begin{definition}\label{def:minimizer}
        For any $x, x'\in\X$ and $a\in[0,1]$, let $t^*(a) := \inf\{s\in\R| s \in \arg\min_{t\in\R} a\ell(t)$ $+(1-a)\ell(-t)\}$ be the infimum of all the minimizers of problem \eqref{eq:d_rho} with conditional probability $a$. We define $d_\rho(x,x') = t^*(\eta(x,x'))$. Then $d_\rho(x,x) \le d_\rho(x, x')$ if $t^*(\eta(x,x)) \le t^*(\eta(x,x'))$.
    \end{definition}
    \begin{remark}
        For many losses like the modified least squares $\ell(-t) = \max\{1-t,0\}^2$, the hinge loss $\ell(-t) = \max\{1-t, 0\}$, the exponential loss $\ell(-t) = \exp(-t)$, and the logistic regression loss $\ell(-t) = \ln(1 + \exp(-t))$, the minimizer $d_\rho(x,x')$ is unique for  $x,x'\in \X$ \cite{zhang04statistical}.  Definition~\ref{def:minimizer} defines $d_\rho(x,x')$ as the minimum of all the minimizers, and the comparison is for the minimum values of $d_\rho(x,x)$ and $d_\rho(x,x')$. We do not make the assumption of the uniqueness of the minimizer here. 
    \end{remark}
    \begin{assumption}\label{loss}
        The loss function $\ell$ is convex, non-decreasing, and non-negative.
    \end{assumption}
    \begin{remark}
        This assumption is natural in our setting because we often use a convex loss to implement algorithms efficiently.
        For any sample pair $(x,y),(x',y')$, the distance or the similarity $d(x,x')$ is supposed to be small when $\tau(y,y') = 1$, then we expect the loss $\ell(d(x,x'))$ to increase as $d(x,x')$ increases. For the case of $\tau(y,y') = -1$, the monotonicity of $\ell$ can be discussed similarly.
    \end{remark}

    According to \eqref{eq:d_rho}, once the loss \(\ell\) is fixed, the true metric \(d_\rho(x,x')\) is completely determined by the conditional probability \(\eta(x,x')\). It is therefore natural to investigate conditions on \(\eta(x,x')\) under which the property $d_\rho(x,x)\le d_\rho(x,x')$ holds. As shown in \eqref{eta}, when \(\eta(x,x')=Prob\{Y=Y'\mid X=x,X'=x'\}\) is large, the samples \(x\) and \(x'\) are more likely to belong to the same class. In that case, one expects the distance or dissimilarity between them to be small, and hence it is natural to expect that \(d_\rho(x,x')\ge d_\rho(x,x)\) whenever the corresponding conditional probabilities satisfy \(\eta(x,x')\le \eta(x,x)\).

For $\eta\in[0,1]$ and $t\in\R$, we define
        \begin{align}\label{Q_form}
            Q(\eta, t) = \eta\ell(t) + (1 - \eta)\ell(-t).
        \end{align}
    It's obvious that $t^*(\eta) \in \arg\min_{t\in\R} Q(\eta,t)$.
     The following lemma shows that $t^*(\eta)$ is non-increasing on $[0,1]$, which is the key step to prove $d_\rho(x,x) \le d_\rho(x,x')$ for the general case.
    \begin{lemma}\label{lemma_regularity}
        Suppose Assumption \ref{loss} holds. If $0\le \eta_1 \le \eta_2 \le 1$, then $t^*(\eta_2) \le t^*(\eta_1)$. 
    \end{lemma}
    \begin{proof}
        When $\eta_1 = \eta_2$, $\eta_1 = 0$ or $\eta_2 = 1$, the proof is trivial. Then we only consider the strict inequality in the condition such that $0 < \eta_1 < \eta_2 < 1$.

       We first show that for any fixed $\eta\in(0,1)$, if there exists a $t <t_1$ such that $Q(\eta, t) < Q(\eta, t_1)$, then we must have $t^*(\eta)< t_1$. We prove this by contradiction. Suppose $t^*(\eta) = t_1$ or $t^*(\eta) > t_1$. The first case is impossible because it then follows that $Q(\eta, t) < Q(\eta, t^*(\eta))$, which contradicts to the definition of $t^*(\eta)$; for the second case, we have 
        \begin{align*}
            &\frac{t^*(\eta) - t_1}{t^*(\eta) - t}Q(\eta, t) + \frac{t_1 - t}{t^*(\eta) - t}Q(\eta, t^*(\eta))\\ 
            <\ & \frac{t^*(\eta) - t_1}{t^*(\eta) - t}Q(\eta, t_1) + \frac{t_1 - t}{t^*(\eta) - t}Q(\eta, t_1)\\
            =\ & Q(\eta, t_1),
        \end{align*}
        where the inequality uses the assumption $Q(\eta,t_2)<Q(\eta,t_1)$ and the definition of $t^*(\eta)$. Since $\frac{t^*(\eta) - t_1}{t^*(\eta) - t_2}t_2 + \frac{t_1 - t_2}{t^*(\eta) - t_2}t^*(\eta) = t_1$, this strict inequality contradicts to the convexity of $Q(\eta, \cdot)$.
        
        From the above discussion, to prove $t^*(\eta_2) \le t^*(\eta_1)$, we just have to show that there exists a $t < t^*(\eta_1)$ such that $Q(\eta_2, t) < Q(\eta_2, t^*(\eta_1))$ or $t^*(\eta_1)$ is also a minimizer of $Q(\eta_2,\cdot)$. We are interested in the behavior of $\ell$ at point $t$ when $t$ is on the left side of $t^*(\eta_1)$. Since $\ell$ is non-decreasing, the possibilities of the behavior can be divided into the following two cases.
        
        {\noindent\bfseries Case 1:} There exists a $t < t^*(\eta_1)$ such that $\ell(t) = \ell(t^*(\eta_1))$. In this case we have 
        \begin{align*}
            Q(\eta_2, t) - Q(\eta_2, t^*(\eta_1)) &= (1 -\eta_2)(\ell(-t) - \ell(-t^*(\eta_1)))\\
            &= \big((1 - \eta_2)/(1 - \eta_1)\big) \big((1 - \eta_1)(\ell(-t) - \ell(-t^*(\eta_1)))\big)\\
            &= \big((1 - \eta_2)/(1 - \eta_1)\big) \big(Q(\eta_1, t) - Q(\eta_1, t^*(\eta_1))\big)\\
            &< 0,
        \end{align*}
        where the first and thrid equalities use the assumption $\ell(t) = \ell(t^*(\eta_1))$, the last inequality follows from the definition of $t^*(\eta_1)$. Therefore, we conclude that $t^*(\eta_2) < t^*(\eta_1)$.

        {\noindent\bfseries Case 2:} There exists a $t < t^*(\eta_1)$ such that $\ell(t) < \ell(t^*(\eta_1))$. By convexity of $\ell$, we know $\ell(t_1) < \ell(t^*(\eta_1)) < \ell(t_2)$ for any $t_1, t_2$ such that $t_1 < t^*(\eta_1) < t_2$. Denote by $l_+'(\cdot)$ and $l_-'(\cdot)$ the right and left derivative functions of $\ell(\cdot)$ respectively. (side derivatives must exist for monotonic functions.) The definition of $t^*(\eta_1)$ indicates that $Q(\eta_1, t_1) > Q(\eta_1, t^*(\eta_1))$ and $Q(\eta_1, t_2) \ge Q(\eta_1, t^*(\eta_1))$, where this can be rewritten in the following inequalities
        \begin{align*}
            \frac{\ell(-t_1) - \ell(-t^*(\eta_1))}{\ell(t^*(\eta_1)) - \ell(t_1)} > \frac{\eta_1}{1 - \eta_1} \ge \frac{\ell(-t_2) - \ell(-t^*(\eta_1))}{\ell(t^*(\eta_1)) - \ell(t_2)}.
        \end{align*}
        Divide both the numerators and denominators of the left hand side and the right hand side of the above inequality by $t^*(\eta_1) - t_1$ and $t^*(\eta_1) - t_2$ respectively. By taking limit as $t_1 \to t^*(\eta_1)_-$ and $t_2 \to t^*(\eta_1)_+$, we get
        \begin{align*}
            \frac{{\ell_+'}(-t^*(\eta_1))}{{\ell_-'}(t^*(\eta_1))} \ge \frac{\eta_1}{1 - \eta_1} \ge \frac{{\ell_-'}(-t^*(\eta_1))}{{\ell_+'}(t^*(\eta_1))}.
        \end{align*}
        According to the convexity of $\ell$ and the fact $\ell(t_1) < \ell(t^*(\eta_1))$, we know that $0<l_-'(t^*(\eta_1)) \le l_+'(t^*(\eta_1))$. Then the above inequality is well defined. Since $\eta_2 > \eta_1$, we know $\frac{\eta_2}{1 - \eta_2} > \frac{\eta_1}{1-\eta_1}$. If $\frac{\eta_2}{1 - \eta_2} > \frac{{\ell_+'}(-t^*(\eta_1))}{{\ell_-'}(t^*(\eta_1))}$, then by the definition of side derivative, there exists a $t < t^*(\eta_1)$ such that $\frac{\eta_2}{1-\eta_2} > \frac{\ell(-t) - \ell(-t^*(\eta_1))}{\ell(t^*(\eta_1)) - \ell(t)}$, which is equivalent to $Q(\eta_2, t) < Q(\eta_2, t^*(\eta_1))$, hence we conclude $t^*(\eta_2) < t^*(\eta_1)$. If $\frac{\eta_2}{1 - \eta_2} \in (\frac{\eta_1}{1-\eta_1},\frac{{\ell_+'}(-t^*(\eta_1))}{{\ell_-'}(t^*(\eta_1))}] \subset [\frac{{\ell_+'}(-t^*(\eta_1))}{{\ell_-'}(t^*(\eta_1))}, \frac{{\ell_-'}(-t^*(\eta_1))}{{\ell_+'}(t^*(\eta_1))}]$, this is equivalent to that $t^*(\eta_1)$ is also a minimizer of $Q(\eta_2, \cdot)$, hence we conclude $t^*(\eta_2) \le t^*(\eta_1)$.
        The proof is then completed.
    \end{proof}\jd
    
    Lemma~\ref{lemma_regularity} also has some implications for binary classification. For a binary classification problem, we aim to learn a classifier $f :\X \to \{-1,1\}$ from a given sample $S$. The performance of a classifier $f$ is usually measured by its generalization error $\bE_Z[\ell(-Yf(X))]$ with a loss $\ell$ that is convex, non-decreasing, and nonnegative (examples of losses can be found in the remark below Definition \ref{def:minimizer}). Similar to the arguments of \eqref{eq:d_rho}, one can derive that, for almost every $x\in\X$, the true predictor (target function) $f_\rho$ with loss $\ell$ can be written as $f_\rho(x) = \arg\min_{t\in\R} p(x)\ell(-t) + (1-p(x))\ell(t)$, where $p(x) = Prob\{Y=1|X=x\}$ is the conditional probability. Therefore, by setting $Q(p,t) = p\ell(t) + (1-p)\ell(-t)$, we get $f_\rho(x) = -\arg\min_{t\in\R} Q(p(x), t) = -t^*(p(x))$ for almost every $x\in\X$. Note Lemma~\ref{lemma_regularity} shows that $t^{*}(\cdot)$ is non-increasing on $[0,1]$. Then we know that the value of the true predictor $f_\rho(x) = -t^*(p(x))$ increases as the conditional probability $p(x)$ increases, which implies that the tendency of the sample $x$ being classified to class $\{1\}$ by the true predictor is also increasing.

    The following property can be directly obtained by using Lemma~\ref{lemma_regularity}. 
    \begin{proposition}\label{regularity}
        Suppose Assumption \ref{loss} holds. If $\eta(x,x') \le \min\{\eta(x,x), \eta(x',x')\}$, then
        \begin{align*}
            \max\{d_\rho(x, x), d_\rho(x', x')\} \le d_\rho(x,x').
        \end{align*}
    \end{proposition}
    \begin{proof}
        Let $\eta_1 = \eta(x,x'), \eta_2 = \eta(x,x), \eta_3 = \eta(x',x')$, then $d_\rho(x,x') = t^*(\eta_1), d_\rho(x,x) = t^*(\eta_2)$ and $d_\rho(x',x') = t^*(\eta_3)$. By Lemma \ref{lemma_regularity}, we get the desired result immediately. 
    \end{proof} 
    \begin{remark}
         Define $\calP = \{P_x \in \R^m | x \in \X\}$ as the conditional distribution family. One can easily observe that $\calP$ is a subset of the probability simplex
    \begin{align*}
        \Delta_m = \left\{(p_1,\ldots,p_m) \in\R^m \left. \right | \sum_{i = 1}^m p_i = 1, p_i \ge 0\right\}.
    \end{align*}
    The condition in Theorem~\ref{regularity} can be written as $\langle P_x, P_{x'}\rangle \le \min\{\|P_x\|^2_2, \|P_{x'}\|^2_2\}$ for almost $x, x' \in \X$. In the aspects of geometry, it requires that the inner product of any two probability vectors in $\calP$ is less than or equal to both the square of their Euclidean norms. One example satisfying this condition is that $\calP$ is a subset of the intersection of the probability simplex $\Delta_m$ and the ball centered at the origin with radius $r \in [0,1]$. In this case, the equality of the condition holds only when $P_x = P_{x'}$.

    Furthermore, $\calP$ has an empty interior in the sub-topology of probability simplex $\Delta_m$. Otherwise, we can find a small enough vector $\epsilon$ such that $P_x + \epsilon \in \calP$ and $\langle P_x, P_x +\epsilon \rangle > \min\{\|P_x\|^2, \|P_x + \epsilon\|^2\}$, where $P_x$ is a relative interior point of $\calP$. 
    \end{remark}

\section{Experiments}\label{sec:exper}
{ 
To empirically validate our theoretical findings, we compare the canonical distance-based criterion
\[
\|f(x)-f(x')\|_2 - b,
\]
which is commonly used in deep metric learning (DML), with our proposed structured model $d(x,x')$ defined in \eqref{eq:structured-NNs}. In both cases, the underlying embedding network is a three-layer MLP producing an embedding $f(x)\in\R^m$.

For our model, we use a practical instantiation of the structured metric in \eqref{eq:structured-NNs}. Specifically, we transform the MLP output into a nonnegative $\ell_1$-normalized representation
\[
g(x):=\frac{1}{\|f(x)\|_1}\big(|(f(x))_1|,\ldots, |(f(x))_m|\big)\in\Delta_m,
\]
where $(f(x))_i$ denotes the $i$-th component of $f(x)$ and $\Delta_m$ is the probability simplex. This normalized vector makes the learned representation naturally interpretable as a probability-vector surrogate of the theoretical construction. Our model is then given by
\[
d(x,x')
=
F_a\!\left(1-2\langle g(x),g(x')\rangle\right).
\]
}

\begin{table}[t]
\centering
\caption{Balanced accuracy (bAcc, mean $\pm$ std over 5 runs) on FashionMNIST.}
\label{tab:fashion_pairwise_bacc}
\begin{tabular}{lccc}
\hline
Pair & Difficulty & Ours (bAcc, \%) & DML (bAcc, \%) \\
\hline
0 vs 2 & medium & $\mathbf{93.86 \pm 0.32}$ & $93.67 \pm 0.18$ \\
\hline
0 vs 6 & hard   & $\mathbf{74.87 \pm 0.36}$ & $74.30 \pm 0.14$ \\
\hline
\end{tabular}
\end{table}

\begin{figure}[t]
\centering

\begin{minipage}[t]{0.48\textwidth}
    \centering
    \includegraphics[width=\linewidth]{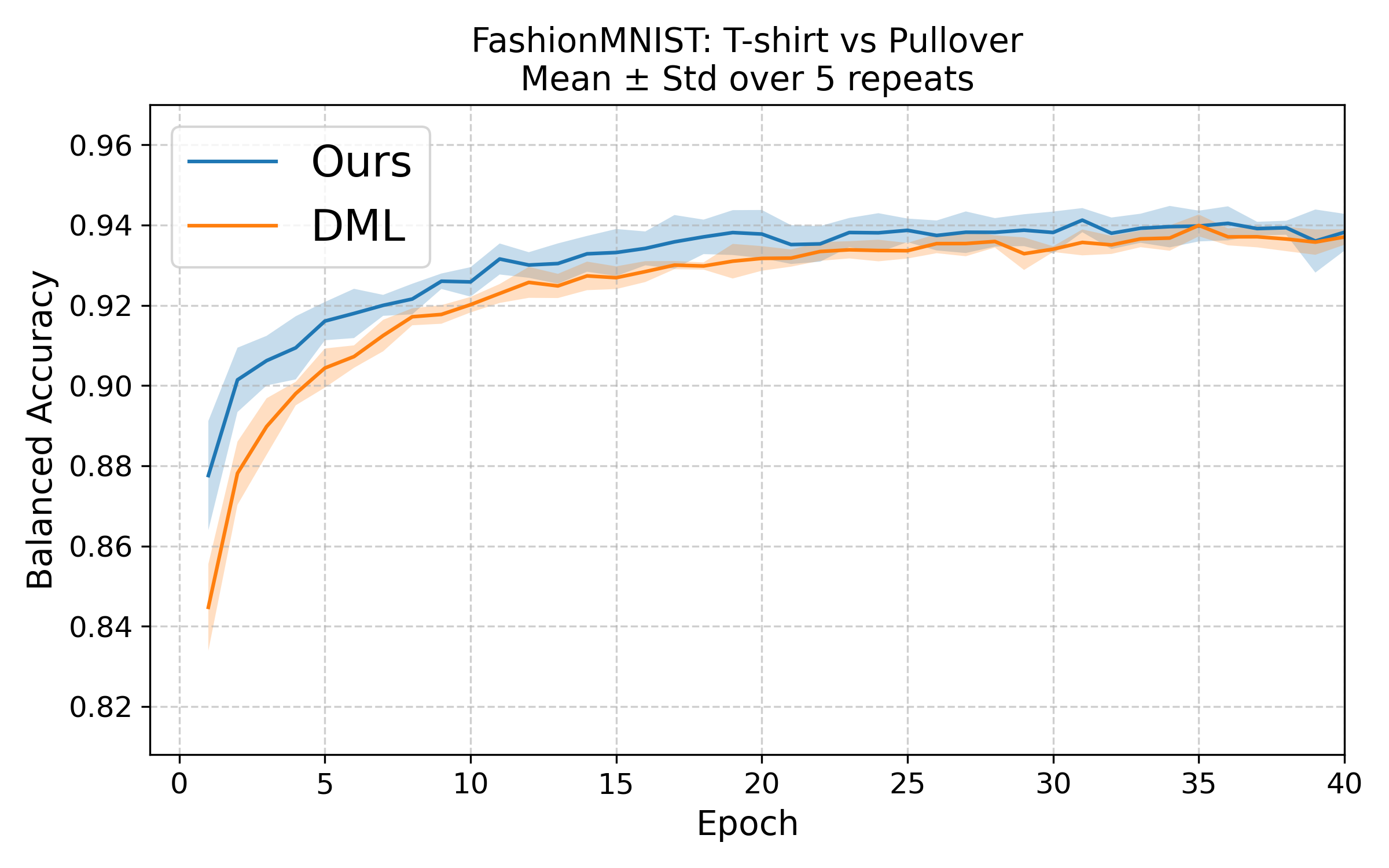}
    
    \vspace{2mm}
    \small (a) FashionMNIST 0 vs 2 (medium)
\end{minipage}
\hfill
\begin{minipage}[t]{0.48\textwidth}
    \centering
    \includegraphics[width=\linewidth]{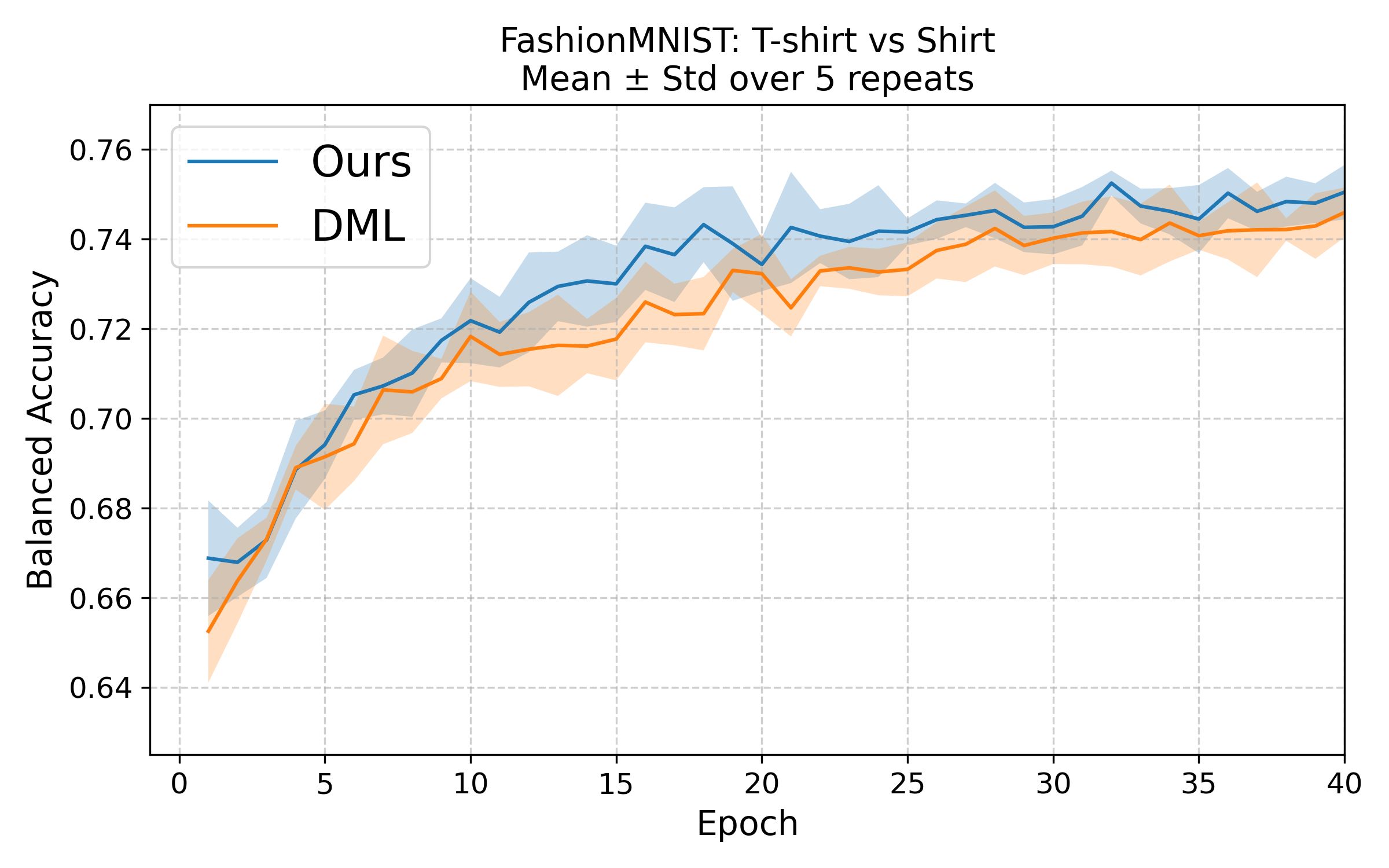}
    
    \vspace{2mm}
    \small (b) FashionMNIST 0 vs 6 (hard)
\end{minipage}

\caption{Epoch-wise balanced accuracy on the same FashionMNIST pairs reported in Table~\ref{tab:fashion_pairwise_bacc}. Curves show mean $\pm$ standard deviation over 5 runs.}
\label{fig:fashion_curves}
\end{figure}

{ 
\subsection{Real-world data}
We evaluate the baseline and our model on binary pairwise metric learning tasks derived from the FashionMNIST dataset. We consider two representative class pairs with different difficulty levels: a medium-difficulty pair (T-shirt/top vs.\ Pullover, classes $0$ vs.\ $2$) and a hard pair (T-shirt/top vs.\ Shirt, classes $0$ vs.\ $6$). For each task, we construct same/different pairs and use balanced accuracy (bAcc) as the evaluation criterion.

\vspace{1mm}

\noindent\textbf{Experimental setup.}
We set $m=2$, which is naturally aligned with the binary-label setting. Both models use a three-layer MLP embedding network $f(x):\R^{28\times 28}\to\R^2$ with hidden width $64$, and are trained using the same hinge loss and the Adam optimizer. The learning rate is set to $10^{-4}$, and each model is trained for $40$ epochs. All reported results are averaged over five independent runs.

\noindent\textbf{Results.}
Table~\ref{tab:fashion_pairwise_bacc} reports the balanced accuracy. On the medium-difficulty pair $0$ vs.\ $2$, both models achieve high performance, with our model slightly outperforming DML ($93.86 \%$ vs.\ $93.67 \%$). On the harder pair $0$ vs.\ $6$, the performance of both models drops substantially, reflecting the increased difficulty of fine-grained discrimination on FashionMNIST. Nevertheless, our model again attains a slightly better final bAcc ($74.87 \%$ vs.\ $74.30 \%$).

Figure~\ref{fig:fashion_curves} shows the epoch-wise balanced accuracy for the same two tasks.
For the medium pair, both models converge to a high-accuracy regime, while our model improves more rapidly in the early epochs and maintains a small but consistent advantage throughout most of training. 
For the hard pair, the overall learning curves are lower, but the same qualitative behavior remains: our model stays slightly above the DML baseline for most epochs and yields a modestly better final accuracy.
Overall, these results show that the proposed structured metric is empirically competitive with a standard deep metric baseline under matched MLP architectures, while also clearly reflecting the difficulty gap between medium and hard pairs.
}

\begin{table}[t]
\centering
\caption{Balanced accuracy (bAcc, \%, mean $\pm$ std over 5 runs) on the synthetic pairwise task. 
Left: varying the uncertain-region mass $\lambda \in \{0.2,0.3,0.4,0.5\}$ with $m=10$. 
Right: varying the number of classes $m \in \{2,3,5,10\}$ with $\lambda=0.5$.}
\label{tab:simulation_bacc}

\begin{tabular}{lcccc}
\hline
 & $\lambda=0.2$ & $\lambda=0.3$ & $\lambda=0.4$ & $\lambda=0.5$ \\
\hline
Ours & $\mathbf{99.93 \pm 0.03}$ & $\mathbf{99.95 \pm 0.04}$ & $\mathbf{99.93 \pm 0.05}$ & $\mathbf{99.92 \pm 0.05}$ \\
\hline
DML  & $96.01 \pm 1.87$ & $89.94 \pm 0.23$ & $83.11 \pm 0.13$ & $75.14 \pm 0.16$ \\
\hline
\end{tabular}

\vspace{8pt}

\begin{tabular}{lcccc}
\hline
 & $m=2$ & $m=3$ & $m=5$ & $m=10$ \\
\hline
Ours & $\mathbf{93.36 \pm 0.11}$ & $\mathbf{99.24 \pm 0.05}$ & $\mathbf{99.91 \pm 0.03}$ & $\mathbf{99.95 \pm 0.05}$ \\
\hline
DML  & $92.28 \pm 0.20$ & $75.41 \pm 0.24$ & $74.99 \pm 0.37$ & $74.88 \pm 0.22$ \\
\hline
\end{tabular}
\end{table}
\subsection{Synthetic data} { 
To empirically illustrate the mechanism behind our theoretical construction, we evaluate our structured model $d(x,x') = F_a\!\left(1-2\langle g(x),g(x')\rangle\right)$ against the canonical distance-based criterion $\|f(x)-f(x')\|_2 - b$ in a controlled setting.
The goal of this experiment is not to maximize raw predictive performance, but to examine whether the two model classes behave differently when the true metric is governed by an inner-product structure rather than by a standard distance threshold.

\vspace{1mm}

\noindent\textbf{Synthetic construction.}
This construction is inspired by the discussion after Proposition~3.3, in particular by the possibility that the true metric may violate standard distance-based intuition.
We generate one-dimensional samples from two regions, denoted by $A$ and $B$.
We assign a probability mass of $\lambda \in (0,1)$ to the uncertain region $B$, with the remaining $1-\lambda$ mass assigned to $A$.
Each region is associated with a specific probability vector:
\[
P^+ = e_m,
\qquad
P^- = \Bigl(p,\frac{1-p}{m-1},\ldots,\frac{1-p}{m-1}\Bigr),
\]
where $e_m$ is the $m$-th standard basis vector and we fix $p=0.6$.
For a pair $(x,x')$, the target similarity defined in \eqref{eta} is
\[
\eta(x,x') = Prob\{Y = Y'|X = x, X' = x'\} = \langle P(x),P(x')\rangle.
\]
We assign the binary pairwise label $\tau(y,y') = 1$ if $\eta(x,x') \ge 1/2$  and $-1$ otherwise.
This construction ensures that the binary labels are induced by the sign of $1 - 2\eta(x, x')$, ensuring that the underlying tasks are consistent with the structure of our model.
In contrast, a distance-based metric can distinguish separate cross-region pairs, but may fail on the crucial $B$--$B$ pairs when their internal similarity falls below the threshold $1/2$.
More explicitly, for any $x_A\in A, x_B \in B$, the three canonical pair types satisfy
\[
\eta(x_A,x_A) = \langle P^+,P^+\rangle = 1,\quad
\eta(x_A,x_B) = \langle P^+,P^-\rangle = \frac{1-p}{m-1}\]\text{ and } 
\[\eta(x_B,x_B) = \langle P^-,P^-\rangle
= p^2 + \frac{(1-p)^2}{m-1}.
\]
Hence, for $p=0.6$, we always have $\eta_{AA}>1/2$ and $\eta_{AB}<1/2$, so $A$--$A$ are similar pairs while $A$--$B$ are different pairs.
However, the status of $B$--$B$ pairs depends on $m$:
when $m=2$, $\eta_{BB}=0.52>1/2$, so $B$--$B$ are similar pairs;
when $m\ge 3$, $\eta_{BB}<1/2$, so $B$--$B$ are different pairs.
This transition explains why the synthetic task is benign for the distance-based baseline when $m=2$, but becomes increasingly mismatched with a purely distance-based metric once $m\ge 3$.

In addition, since the uncertain region $B$ has probability mass $\lambda$, the proportion of $B$--$B$ pairs is exactly $\lambda^2$.
Therefore, once $m\ge3$, increasing $\lambda$ primarily increases the frequency of the structurally difficult pairs.
This explains why the degradation of the distance-based DML baseline is closely tied to $\lambda^2$ in our experiments.

\vspace{1mm}

\noindent\textbf{Experimental setup.}
Both models utilize a three-layer MLP with hidden width $16$.
We conduct two sets of experiments: (i) a  $\lambda$-sweep, where we fix $m=10$ and vary $\lambda \!\in\! \{0.2, 0.3, 0.4, 0.5\}$ to observe the impact of increased uncertain region, and (ii) an $m$-sweep, where we fix $\lambda=0.5$ and vary $m \in \{2, 3, 5, 10\}$ to test the effect of the number of labels.
Models are trained using the same hinge loss and Adam optimizer for $80$ epochs.
We report the bAcc averaged over five independent runs.

\vspace{1mm}

\noindent\textbf{Results for varying $\lambda$.}
The upper half of Table~\ref{tab:simulation_bacc} reports the balanced accuracy for different values of $\lambda$.
Our model achieves near-optimal performance across the whole range, with bAcc around $99.9\%$.
In contrast, the DML baseline deteriorates monotonically as $\lambda$ increases, dropping from $96.01$ at $\lambda=0.2$ to $75.14$ at $\lambda=0.5$.
As shown in Figure \ref{fig:lam^2-bacc}, the bAcc of the DML baseline decreases approximately linearly with $\lambda^2$.
This observation is consistent with our theoretical prediction: as $\lambda$ grows, the proportion of $B$--$B$ pairs increases, and these are precisely the pairs that can not be captured by a simple distance-based model.}

\vspace{1mm}

\begin{figure}
    \centering
    \includegraphics[width=0.5\linewidth]{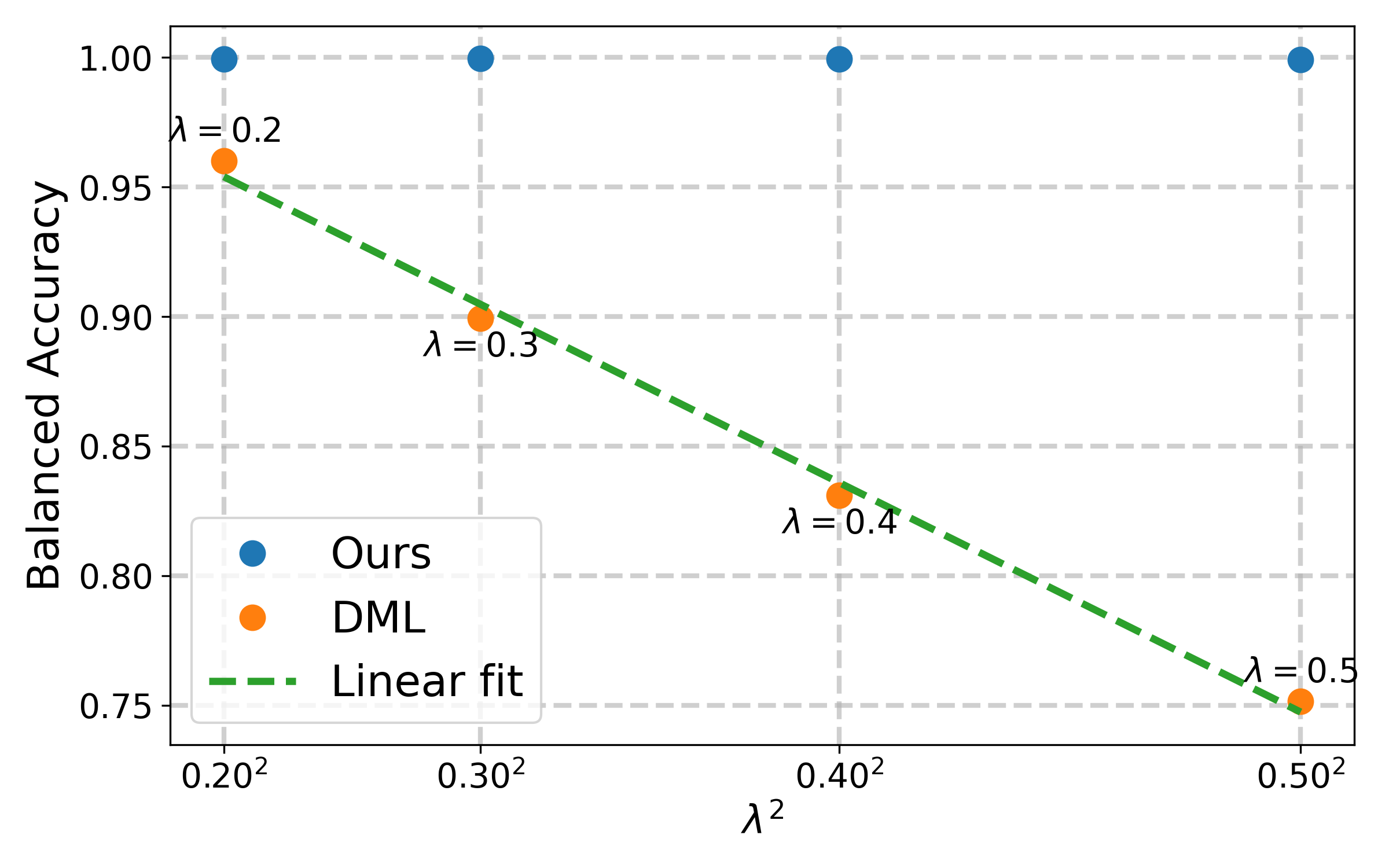}
    \caption{Balanced accuracy versus $\lambda^2$ on the metric learning task with $m=10$. Each point is annotated by the corresponding value of $\lambda.$}
    \label{fig:lam^2-bacc}
\end{figure}
{ 
\noindent\textbf{Results for varying $m$.}
The lower half of Table~\ref{tab:simulation_bacc} reports the balanced accuracy for different values of $m$ at $\lambda=0.5$. When $m=2$, the two models perform similarly ($93.36$ for ours versus $92.28$ for DML). This is expected since the similarity in region $B$ satisfies
\[
\eta(x_B,x_B) = \|P^-\|_2^2 = p^2 + \frac{(1-p)^2}{m-1} = 0.52 > \tfrac12,
\]
so the $B$--$B$ pairs do not violate the distance-based structure.
However, as $m \ge 3$, we have
\[
\|P^-\|_2^2 < \tfrac12,
\]
and the discrepancy becomes pronounced. Our model rapidly approaches near-perfect accuracy ($>99.9\%$) while DML drops sharply to about $75\%$ and saturates.
This clear gap is consistent with our theory.
When the number of classes becomes larger, a simple distance-based metric is less able to describe the target pairwise similarity.

\vspace{1mm}

\noindent\textbf{Discussion.}
These synthetic results are in line with our theory.
Our model stays accurate when $\lambda$ or $m$ increases, since it is built to match the probabilistic structure of the target similarity.
In contrast, although the DML baseline performs well in simpler cases, its accuracy drops once the probabilistic structure can no longer be captured by a standard distance-based metric.

}

%% file: conclusion.tex
    \section{Conclusion}\label{sec:conclu}
    In this paper, we provide a comprehensive generalization analysis for metric and similarity learning with the hinge loss. By deriving an explicit representation of the true metric \(d_\rho\) under the hinge loss, we construct a novel hypothesis space consisting of structured deep ReLU networks, and establish excess risk bounds by carefully controlling both the approximation error and the estimation error within this structured hypothesis space. In particular, we derive an explicit learning rate, up to a logarithmic factor, $O\big(n^{-\frac{(\theta + 1) r}{d + (\theta + 2)r}}\big)$. 
We also revisit several structural properties of the problem setting and the true metric under a general loss function.  Experiments show that the proposed structured model is empirically competitive on real-world data and better captures the underlying similarity mechanism in the synthetic settings studied in this paper.  An interesting direction for future work is to extend our results to the logistic loss, whose target function is unbounded \cite{zhang2024classification}.